\documentclass[letterpaper, 10 pt, conference]{ieeeconf}  

\IEEEoverridecommandlockouts                              

\usepackage{amsfonts, dsfont, amsmath, graphicx, algorithm, algorithmic, color} 
\DeclareMathOperator*{\argmax}{\arg\!\max}
\newtheorem{lemma}{Lemma}
\newtheorem{remark}{Remark}
\newtheorem{corollary}{Corollary}

\title{\LARGE \bf
Transition-based versus State-based Reward Functions for MDPs \\ with Value-at-Risk*
}

\author{Shuai Ma$^{1}$ and Jia Yuan Yu$^{1}$
\thanks{*This work was supported in part by the scholarship from China Scholarship Council (CSC).}
\thanks{$^{1}$Shuai Ma and Jia Yuan Yu are with Concordia Institute of Information System Engineering, Faculty of Engineering and Computer Science,
        Concordia University, 1515 Ste. Catherine St. West, Montreal, Quebec, Canada.
        {\tt\small m\_shua@encs.concordia.ca} and {\tt\small jiayuan.yu@concordia.ca}}%
}

\begin{document}

\maketitle
\thispagestyle{empty}
\pagestyle{empty}

\begin{abstract}
In reinforcement learning, the reward function on current state and action is widely used. When the objective is about the expectation of the (discounted) total reward only, it works perfectly. However, if the objective involves the total reward distribution, the result will be wrong. 
\color{black}
This paper studies Value-at-Risk (VaR) problems in short- and long-horizon Markov decision processes (MDPs) with two reward functions, which share the same expectations. 
Firstly we show that with VaR objective, when the real reward function is transition-based (with respect to action and both current and next states), the simplified (state-based, with respect to action and current state only) reward function will change the VaR. Secondly, for long-horizon MDPs, we estimate the VaR function with the aid of spectral theory and the central limit theorem. Thirdly, since the estimation method is for a Markov reward process with the reward function on current state only, we present a transformation algorithm for the Markov reward process with the reward function on current and next states, in order to estimate the VaR function with an intact total reward distribution.  

\end{abstract}

\section{INTRODUCTION}

A Markov decision process (MDP) is a mathematical framework for formulating the discrete time stochastic control problems. This framework has two features, one is randomness, mainly reflected by transition probability, the other one is controllability, reflected by policy. These two features enable MDP as a natural tool in sequential decision-making for practical problems.

An optimal policy describes how to make decisions sequentially, in order to maximize (minimize) some objective concerning the total reward. The objective is a measure of the random reward sequence set for a given MDP.
\color{black}
The typical objective class concerns the expected total reward 
over a potentially infinite horizon \cite{Derman:1970:FSM:578852,Puterman1994a}.
\color{black}

However, the objectives involving expectation are not sufficient for many risk-averse problems, where the risk concerns arise not only mathematically but also psychologically. A classic example in psychology is the ``St. Petersburg Paradox,'' which refers to a lottery with an infinite expected reward, but people only prefer to pay a small amount to play. This problem is thoroughly studied in utility theory, and a recent study brought this idea to reinforcement learning \cite{Prashanth2015}.
A more mathematical example would be autonomous vehicles, in which a sufficient safety factor is more important than the optimal expectation. In general, when high reliability is concerned, the objective should be formulated as probability instead of expectation.

Two risk-sensitive objective classes have been widely examined in recent years. 
One is the coherent risk measure \cite{Artzner1998a}, which has a set of intuitively reasonable properties (convexity, for example). A thorough study in coherent risk optimization can be found in \cite{Ruszczynski2006a}. 
The other important class is the mean-variance measure \cite{D.J.White1988a,doi:10.1287/opre.42.1.175,Mannor2011a}, in which the expected return is maximized with a given risk level (variance). It is also known as modern portfolio theory.   

Value-at-Risk (VaR) originates from finance. For a given portfolio (an MDP with a policy), a loss threshold (target level), and a time-horizon, VaR concerns the probability that the loss on the portfolio exceeds the threshold over the time horizon. VaR is hard to deal with since it is not a coherent risk measure \cite{Riedel2004}.

In reinforcement learning, we encounter two types of reward functions. In many situations the real reward function is $r:S \times A \times S \rightarrow \mathbb{R}$ , and here we call it \textit{transition-based} reward function. When the real reward function is transition-based, it is usually simplified to $r':S \times A \rightarrow \mathbb{R}$, and here we call it \textit{state-based} reward function. State-based reward function is widely used, such as in Q-learning \cite{watkins1992q} and linear programming for constrained MDP \cite{altman1999constrained}. However, when the objective concerns the total reward distribution instead of the expectation only, the simplification of the reward function will change the optimal value. Imagine that a robot is learning to move around in a restricted area, and each move is attached with a reward, which should involve action, current and next states. If Q-learning is directly applied to this agent, and a reward function is learned with current state and action only, then the result with respect to a risk-sensitive objective, which considers the total reward distribution instead of the expectation only, will be wrong.


In this paper, we analyze the risk-sensitive MDPs and consider two VaR objectives for example, and propose a transformation for keeping the total reward distribution intact while simplifying the reward function. The analysis and the transformation method are illustrated in an inventory control problem, which is modelled as MDPs in both short and long horizons with finite state and action spaces. 
In Section 2, the notations of MDP and two VaR problems are introduced.
\color{black}
In Section 3, we use a short-horizon MDP to illustrate that with the expected total reward objective, MDPs with the two reward functions have the same optimal expectation/policy, but different total reward distributions, which result in different VaRs. We also compare the augmented-state 0-1 MDP method and the proposed method for VaR problems, and give the proof for the first method. 

Our main contributions are described and discussed in Section 4, which include the following:
\begin{itemize}

\item We propose a state-transition transformation algorithm for Markov reward processes derived from MDPs with transition-based reward functions, in order to estimate the total reward distribution. Since the CDF estimation method is for a Markov reward process with a reward function $r'_{\pi}:S \rightarrow \mathbb{R}$, the proposed algorithm can transform a Markov reward process with a reward function $r_{\pi}:S \times S \rightarrow \mathbb{R}$ for the CDF estimation method, and keep the distribution intact.
\item We illustrate that both VaR problems relate to the VaR function, which is the infimum of the total reward distribution set. We estimate the VaR function with the aid of spectral theory and the central limit theorem for long-horizon MDPs.
\end{itemize}
In short, when the objective is risk-sensitive, it often refers to the total reward distribution instead of the expectation only, and if the real reward function is transition-based, it should not be simplified. 
In order to keep the distribution intact while being able to implement some technique which requires state-based reward functions, the proposed transformation algorithm should be carried out first. 
\color{black}
For related studies which concerned VaR or other risk-sensitive problems, we believe that they should be revisited using our proposed transformation approach instead of simplifying the reward directly.

\section{PRELIMINARIES AND NOTATIONS}
A finite-horizon MDP,   
$$\langle N, S, A, r, p, \mu, v\rangle,$$
is observed at decision epochs $\{0, 1, \cdots, N\}$ and $N$ is a finite positive integer; 
$S$ is a finite state space, and $X_t$ denotes the state at epoch $t$;
$A_x$ is the legitimate action set for $x \in S$, $A = \bigcup_{x \in S}A_x$ is a finite action space, and  $K_t$ denotes the action at epoch $t$; 
$r: S \times A \times S \rightarrow \mathbb{R}$ is the bounded and measurable reward function, and $r(x, a, y)$ denotes the reward (or cost if negative), given $X_t = x$,  $X_{t+1}=y$, $x,y \in S$, and the action $K_t=a \in A_x$.  we say this reward function is \textit{transition-based};
$p(y \mid x, a) = \mathbb{P}(X_{t+1}=y \mid X_t = x, K_t = a)$ denotes the homogeneous transition probability; 
$\mu$ is the initial state distribution;
$v$ denotes the salvage function. 

The optimal policy $\pi^*$ is determined by the objective. A policy $\pi$ refers to a sequence of decision rules ($\pi_0, \pi_1, \cdots,\pi_{N-1}$). Different types of decision rule are used in different situations, and here we focus on deterministic Markovian decision rules, which chooses an action $a \in A_x$ given the current state $x \in S$.

In a finite-horizon MDP with an expectation objective \cite{Puterman1994a}, for all $x \in S$, $a \in A_x$, the transition-based reward function is usually simplified by  

$$
r'(x,a) = \sum_{y \in S}r(x,a,y)p(y \mid x, a) \eqno{(1)}
$$


Here we say the reward function $r'$ is \textit{state-based}. It is suitable to simplify the reward function when the objective concerns the total reward expectation only, but when the objective is risk-sensitive (with respect to the total reward distribution), it will miss the optimal value and policy. 

In this paper we take VaR as an example to show the effect of the reward simplification on risk-sensitive objectives. 
Two VaR problems described in \cite{Filar1995b} are considered as possible objectives. 
Since there exists an optimal deterministic policy for finite-horizon MDPs with VaR objectives \cite{Wu1999a}, we only consider the deterministic policy space $\Pi^N$. Given a policy $\pi \in \Pi^N$ and an initial distribution $\mu$, we have the total reward $\Phi^{\pi}_{\mu} = \sum_{t=0}^{N-1} r(X_t,\pi(X_t),X_{t+1}) + v(X_N)$, where $X_0 \sim \mu$. To simplify the notation we henceforth denote the total reward by $\Phi$. Denote $F^{\pi}_{\Phi}$ as the total reward CDF with the policy $\pi$. VaR addresses the following problems.  

\newtheorem{definition}{Problem}
\begin{definition}
	Given a quantile $\alpha \in [0,1]$, find the optimal threshold $\rho_{\alpha} =  \sup\{\tau \in \mathbb{R}\mid \mathbb{P}(\Phi > \tau) \ge \alpha, \pi \in \Pi^N\}=\sup\{\tau \in \mathbb{R}\mid F^{\pi}_{\Phi}(\tau) \le 1-\alpha, \pi \in \Pi^N\}$.
\end{definition}	

This problem refers to the quantile function, i.e., $F^{\pi -1}_{\Phi}$.

\begin{definition}
	Given a threshold $\tau \in \mathbb{R}$, find the optimal quantile $\eta_{\tau} = \sup\{\alpha \in [0,1]\mid F^{\pi}_{\Phi}(\tau) \le 1-\alpha, \pi \in \Pi^N\}$.
\end{definition}

This problem concerns $F^{\pi}_{\Phi}$. 
Both VaR problems relate to $\inf\{F^{\pi}_{\Phi} \mid \pi \in \Pi^N\}$, and here we call it \textit{VaR function}. 
As will be illustrated below, when the horizon is short (Section 3), any point along the VaR function is $(\tau, 1-\eta_{\tau})$, and when the horizon is long (Section 4), and every (estimated) $F^{\pi}_{\Phi}$ is strictly increasing, 
any point along the VaR function is $(\rho_{\alpha}, 1-\eta_{\tau})$ with $\tau = \rho_{\alpha}$ or $\alpha = 1-\eta_{\tau}$. 

The state-based reward function is commonly used in most MDP studies even considering risk (\cite{Filar1995b} for example). However, if the real reward function transition-based and the objective is risk-sensitive, the simplification with (1) will miss the optimality, i.e., neither the policy nor the VaR is optimal. 
Here we use a short-horizon inventory control problem to illustrate the effect of reward function on the optimality with two VaR objectives, and how the VaR function solves both VaR problems.
\color{black}

\section{SHORT-HORIZON MDP FOR INVENTORY PROBLEM}

The inventory control problem is a straightforward example for interpreting the effect of the two reward functions, since the reward (sales volume) is related to both current and next states (inventory level). In this section,
we illustrate that both VaR problems refer to the VaR function, and show how the reward simplification affects the VaR function. 
\color{black}
 

\subsection{MDP Description and Expected Total Reward} 
Section 3.2.1 in \cite{Puterman1994a} described the model formulation and some assumptions for a single-product stochastic inventory control problem. Define the warehouse capacity $M \in \mathbb{N^+}$, and the state space $S=\{0, \cdots, M\}$. Briefly, at time epoch $t \in \{0, \cdots, N\}$, define $X_t$ as the inventory level before the order, $K_t \in \{0, \cdots, M-X_t\}$ as the order quantity, $D_t$ as the demand with a time-homogeneous probability distribution $\mathbb{P}(D_t=i)$, where $i \in \{0, \cdots, M\}$, then we have $X_{t+1}=max\{X_t+K_t-D_t,0\}$. 

For $x \in S$, define $c(x)$ as the cost to order $x$ units, and a fixed cost $W \ge 0$ for placing orders, then we have the order cost $O(x) = (W+c(x)) \mathds{1}_{[x>0]}$.
$f(x)$ denotes the revenue when $x$ units of demand is fulfilled. Then we have the real reward function $r(X_t,K_t,X_{t+1}) = f(X_t+K_t-X_{t+1})-O(K_t)$, which is transition-based. Here we ignore the maintenance fee to simplify the problem. 

We set the parameters as follows. The time horizon $N=2$, the fixed order cost $W=4$, the variable order cost $c(x)=2x$, the salvage reward $v=X_N$, the warehouse capacity $M=2$, and the price $f(x)=8x$.  The probabilities of demands are $\mathbb{P}(D_t=0)=0.25$, $\mathbb{P}(D_t=1)=0.5$, $\mathbb{P}(D_t=2)=0.25$ respectively. Initial distribution $\mu(0) = 1$, i.e., $X_0=0$. Firstly we calculate the real reward function by $r(X_t,K_t,X_{t+1}) = f(X_t+K_t-X_{t+1})-O(K_t)$. Secondly, we calculate the simplified reward function $r'$ by (1), which is state-based. 

The two reward functions and the transition probabilities are illustrated in Fig.~\ref{MCDesc}. The bold parts are an example which illustrates the difference between the two reward functions. The label $2(0,0.5)$ below the transition from 0 to 1 means that the reward $r(0,2,1)=0$ and the transition probability is 0.5; the label $2(0)$ in the text box near state 0 means when $X_t = 0$ and $K_t=2$, the simplified reward $r'(0,2) = 0$.
Besides, we set the initial state distribution $\mu = [1,0,0]$, and the salvage reward $v(x)=x$, for all $x \in S$. Now we have two MDPs with different reward functions: $\langle N, S, A, r, p, \mu, v\rangle$ and $\langle N, S, A, r', p, \mu, v\rangle$. 

\begin{figure}[b] 
\includegraphics[scale=0.9]{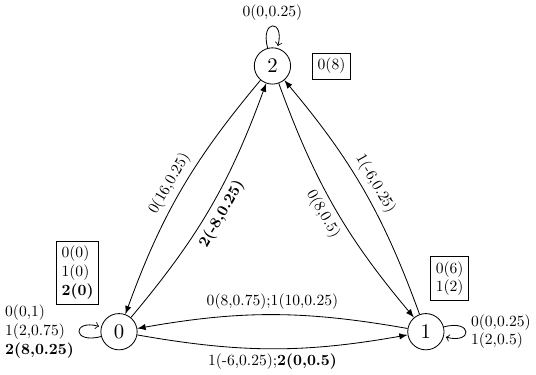} %
\caption{The two reward functions and the transition probabilities for the inventory problem. The labels $a(r(x,a,y),p(y \mid x,a))$ along transitions denote the real transition-based reward function and the transition probability, and the labels $a(r'(x,a))$ in the text boxes near states denote the state-based reward function simplified with (1).} 
\label{MCDesc}
\end{figure}



Without loss of generality, we consider the nominal expected total reward objective (or discounted / average), and solve the preceding MDPs. 
The optimal policy for both MDPs is $\pi^* = (\pi_0,\pi_1)$, where $\pi_0(0) = \pi_1(0) = 2$ and $\pi_1(x) = 0$, for $x \in \{1,2\}$. The expected total reward $\mathbb{E}(\Phi)=6.5625$.
\color{black}
As shown in Fig.~\ref{CDFdiff}, with the expected total reward objective, the simplification of transition-based reward function leads to a different total reward distribution, which leads to a different VaR function shown in Fig.~\ref{PFshort}. 
\begin{figure}[t]
\vskip 0.2in
\begin{center}
\centerline{\includegraphics[width=\columnwidth]{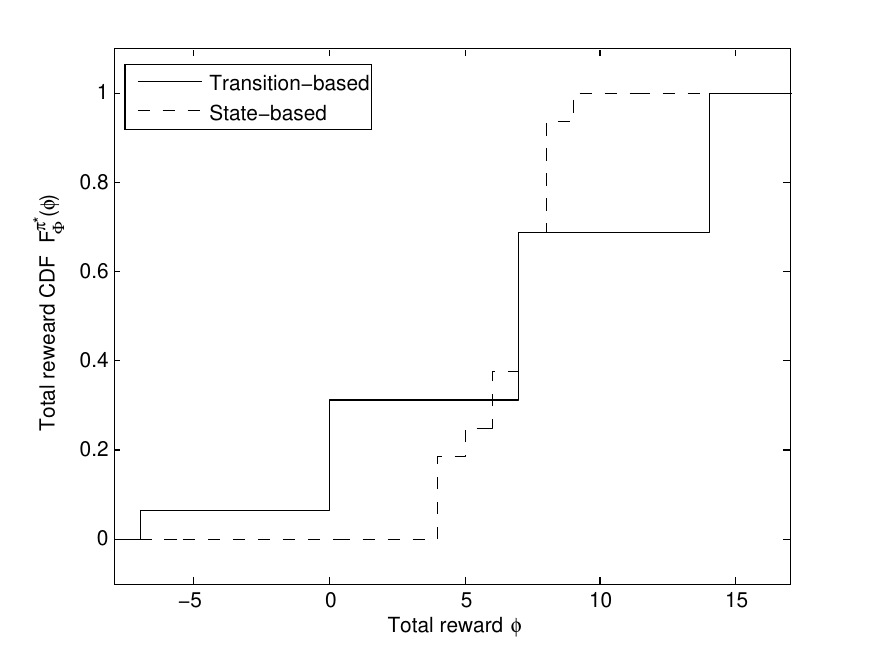}}
\caption{Taking the expected total reward as the objective, the two MDPs, $\langle N, S, A, r, p, \mu, v\rangle$ and $\langle N, S, A, r', p, \mu, v\rangle$, share the same optimal policy $\pi^*$, but the total reward CDFs are different.}
\label{CDFdiff}
\end{center}
\vskip -0.2in
\end{figure}

\subsection{VaR Problem}

Unlike the expected total reward, VaR are not time-consistent, so the backward induction cannot be implemented directly. For short-horizon MDPs, one method is the augmented-state 0-1 MDP \cite{Xu2011}, which incorporate the cumulative reward space in the state space, and brings in the threshold and reorganizes the MDP components, in order to calculate the quantile as expectation.
The other method is to enumerate all policies to achieve the total reward CDF set, then calculate the VaR function (the infimum of the total reward CDF set). Here we describe and compare the two methods.
\color{black}

\subsubsection{Augmented-State 0-1 MDP}

Since the cumulative reward information is needed for the optimality \cite{Bouakiz1995}, an augmented state space is adopted to keep track of it. 
For short-horizon MDPs with VaR Problem 2, Xu and Mannor \cite{Xu2011} presented a state augmentation method to include the cumulative reward in the state space. This state augmentation is also implemented in several former studies \cite{Bouakiz1995,Wu1999a,Ohtsubo2002,Xu2011}.


\begin{lemma}
	For every finite-horizon MDP $\langle N, S, A, r, p, \mu, v\rangle$, there exists an augmented-state 0-1 MDP $\langle N, S', A, v', p', \mu'\rangle$, in which the optimal expected total reward equals to the optimal quantile $\eta_{\tau}$  of the original MDP with a threshold $\tau \in \mathbb{R}$.
\end{lemma}

\begin{proof}
Given an MDP $\langle N, S, A, r, p, \mu, v\rangle$, 
define $C$ as the cumulative reward space, $m_a$ and $M_a$ as the minimum and maximum of the rewards for action $a$. Then $C$ can be $\{0\}\bigcup_{t=1}^N \bigcup_{a \in A}[t \cdot m_a, t \cdot M_a]$, or we can acquire $C$ by enumerating all possibile cumulative reward within the short horizon. Define the augmented state space $S' = S \times C$ for the new MDP.
For all $x,y \in S$ and $x',y' \in S'$, define the action space $A'_{x'} = A_x$ where $x' = (x,\cdot)$; define the transition kernel $p'(y'\mid x',a) = p(y\mid x,a)$ where $y'=(y,c_i), x'=(x,c_j)$, $c_i-c_j = r(x,a,y)$, $c_i,c_j \in C$ and $a \in A_x$; define the initial distribution $\mu'((x,0)) = \mu(x)$. Set the salvage reward $v' = \mathds{1}_{[\Phi \ge \tau-v]}$, where $\tau$ is the threshold and $\Phi$ is the cumulative reward at the final epoch.
Now we have an augmented-state 0-1 MDP $\langle N, S', A, v', p', \mu'\rangle$.

	Given an augmented-state 0-1 MDP $\langle N, S', A, v', p', \mu'\rangle$, for all $x' \in S'$, implement the backward induction as follows.
	
	\textbf{Step 1}: Set $t = N$ and 
	$$u^*_N(x') = r'_N(x') = \mathds{1}_{[\Phi \ge \tau-v]}.$$
	
	\textbf{Step 2}: Set $t = t-1$, and compute $u^*_t(x')$ by
	$$u^*_t(x') = \max_{a \in A_{x'}}\{r'(x',a) + \sum_{y' \in S'}p'(y'\mid x',a)u^*_{t+1}(y')\},$$
	where  $r'(x',a) = 0$, therefore,
	$$u^*_t(x') = \max_{a \in A_{x'}}\{\sum_{y' \in S'}p'(y'\mid x',a)u^*_{t+1}(y')\}.$$
	Since the only rewards are $r'_N = \mathds{1}_{[\Phi \ge \tau-v]}$, we have $u_t(x') = P(\Phi \ge \tau\mid X'_t = x')$, i.e., the probability that the total reward $\Phi \ge \tau$ given any state at any epoch.
	
	\textbf{Step 3}: If $t=1$, stop. Otherwise return to Step 2.
The optimal policy derived from $$A^*_{x',t} = \argmax_{a \in A_{x'}}\{P(\Phi \ge \tau\mid X'_t = x')\}$$
	gives the highest probability to reach the threshold.   
\end{proof}

This 0-1 MDP enables backward induction to solve VaR Problem 2 with calculating the quantile $\eta_{\tau}$ as the expectation. Filar et al. \cite{Filar1995b} used the same ``0-1'' method for infinite-horizon MDPs with both VaR objectives. 
The augmented-state 0-1 MDP algorithm is presented in \cite{Xu2011}. 
In the implementation of the algorithm, it is worth noting that, in most instances, it is more efficient to deal with the state space in a time-dependent way, since at each epoch, only a subspace of $S'$ is feasible. Furthermore, since the reward function is recorded by means of the cumulative reward, the reward simplification still affects the result.


Now we use the augmented-state 0-1 MDP method to solve the inventory control problem described in Section 3.1. We consider the VaR Problem 2 with the threshold $\tau=9$, for instance. The optimal policy for the MDP with the transition-based reward function is $\pi^* = (\pi_0,\pi_1)$, where $\pi_0((0,0)) = 2$, $\pi_1((0,2)) = 2$, $\pi_1((0,8)) = $1 or 2, $\pi_1((1,0)) = 1 \text{ or } 2$, $\pi_1((1,6)) = 0 \text{ or } 1$, and $\pi_1((2,-2)) = 0 \text{ or } 1$. The optimal quantile is $\eta^*_{\tau} = 0.3125$. The optimal policy for the MDP with a simplified reward function is $\pi^* = (\pi'_0,\pi'_1)$, where $\pi'_0((0,0)) = 2$, $\pi'_1((2,0)) = 0$. And the optimal quantile is $\eta^*_{\tau} = 0.1875$. The conclusion drawn in Section III.A, which claims that the reward simplification changes the VaR, is verified here.

As described above, it is clear that the augmented-state 0-1 MDP method is for VaR Problem 2 with a specified threshold only. In order to achieve either VaR problem with any $\tau \in \mathbb{R}$ or $\alpha \in [0,1]$, we enumerate all the deterministic policies on the augmented state space $S'$ to acquire VaR. 

\subsubsection{Policy Enumeration}

In a finite-horizon MDP, the VaR function is a right-continuous step function. It is worth noting that its jump points contains the VaR information. Denote the jump set $\{x_i\}$, i.e., for every $x_i$ in this set, we have $\eta_{x_i^-} < \eta_{x_i}$, and for every $\tau \in \mathbb{R}$, such that $\eta_{\tau^-} < \eta_{\tau}$, we have $\tau \in \{x_i\}$.

\begin{remark}[VaR function in a finite horizon]
Given the jump set $\{x_i\}$ for a finite-horizon MDP, for any $\tau \in \mathbb{R}$, denote $\alpha^* = 1-\eta_{\tau}$, and $\tau^* = \rho_{(1-\alpha^*)}$, then we have 
$(\tau^*, 1-\alpha^*) \in \{(x_i^-, \eta_{x_i^-})\} \cup \{(+\inf, 1)\}$. There exists a similar conclusion for any $\alpha \in [0,1]$. 

\end{remark}
In other words, for a finite-horizon MDP, the set $\{(x_i^-, \eta_{x_i^-})\} \cup \{(+\inf, 1)\}$ contains all solutions to the two VaR problems defined in Section II.A.
\color{black}

\begin{figure}[b]
\vskip 0.2in
\begin{center}
\centerline{\includegraphics[width=\columnwidth]{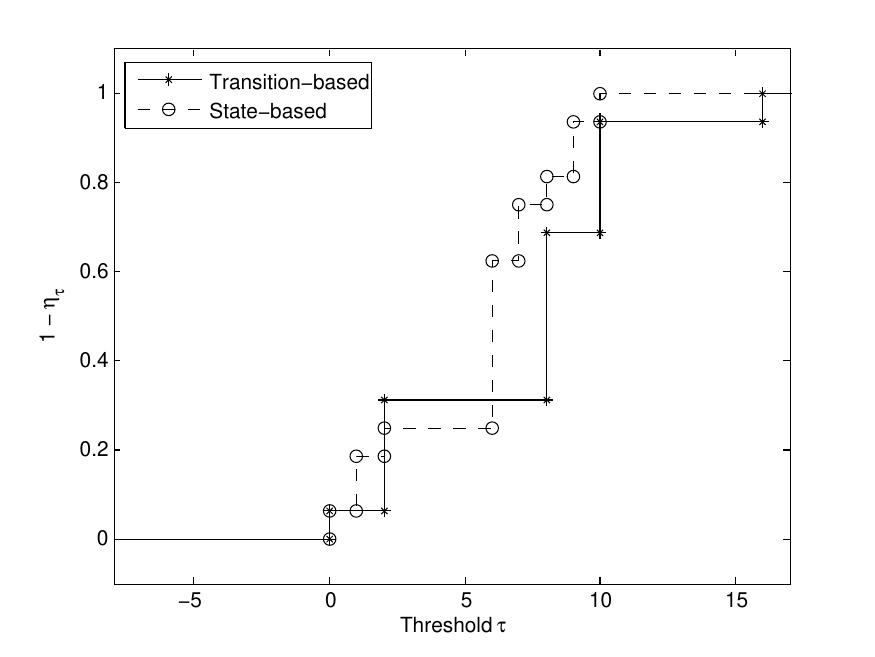}}
\caption{The VaR functions for the short-horizon inventory problems with two reward functions.} 
\label{PFshort}
\end{center}
\vskip -0.2in
\end{figure}

In the same MDP setup described in Section III.A, we study the effect of the reward functions on the VaR function. Given the two MDPs, $\langle N, S, A, r, p, \mu, v\rangle$ and $\langle N, S, A, r', p, \mu, v\rangle$, Fig.~\ref{PFshort} illustrates that the simplification of reward function changes the VaR function. 
Now we can verify the solution to the VaR problem with a specified threshold $\tau=9$. Furthermore, for any threshold $\tau \in \mathbb{R}$, we can acquire $\eta_{\tau}$ along the curves, so as for any quantile. 
For example, when $\tau = 7.5$, $\eta_{\tau}$ for the MDP with a transition-based reward function is 0.6875 ($1-0.3125$), and $\eta'_{\tau}$ for the MDP with a state-based reward function is 0.25 ($1-0.75$). The reward simplification results in a nontrivial difference ($0.6875-0.25=0.4375$), which may have an effect on a risk-sensitive decision making.
\color{black}

In conclusion, the reward simplification changes the VaR. For the VaR Problem 2, the augmented-state 0-1 MDP method enables the backward induction algorithm. However, this method fails for long-horizon MDPs,
and it works for the VaR problem 2 with a specified threshold only. 
Since both VaR problems relate to the VaR function, how to achieve it effectively in a short horizon needs further study.
\color{black}

\section{VAR PROBLEMS IN LONG-HORIZON MDPS}
Since it is intractable to find the exact optimal policy for a long-horizon MDP with a VaR objective, we look for a deterministic stationary policy instead, i.e., $\pi^* \in \Pi$. With the aid of spectral theory and the central limit theorem, the total reward CDF set $\{F^{\pi}_{\Phi}\}$ can be estimated for the MDP with a state-based reward function by enumerating all the deterministic policies, in order to achieve the VaR function.
For MDPs with transition-based reward functions, we present an algorithm to transform a Markov reward process with the reward function $r_{\pi}:S \times S \rightarrow \mathbb{R}$ to one with $r^{\dagger}_{\pi}:S^{\dagger} \rightarrow \mathbb{R}$, where $S^{\dagger} = S \times S$, in order to keep the VaR function (also all the total reward distributions) intact. With a slight abuse of notation, we say the reward function $r_{\pi}$ ($r^{\dagger}_{\pi}$) is transition-based (state-based).
\color{black}
\subsection{Total Reward CDF Estimation}
Firstly we estimate the total reward distribution for a long-horizon Markov reward process derived from an MDP with a state-based reward function. Given an MDP with a state-based reward function $\langle N, S, A, r', p, \mu\rangle$\footnote[5]{Though the salvage function $v$ is ignored when the horizon is long, it can be involved if necessary.} and a deterministic policy $\pi$, we have a Markov reward process $\langle N, S, r'_{\pi}, p_{\pi}, \mu\rangle$. 
For $x,y \in S$, the reward is $r'_{\pi}(x) = r'(x,\pi(x))$, and the transition kernel is $p_{\pi}(x,y)=p(x,\pi(x),y)$.

Kontoyiannis and Meyn \cite{kontoyiannis2003} proposed a method to estimate $F^{\pi}_{\Phi}$. In a positive recurrent Markov reward process with invariant probability measure (stationary distribution) $\xi$, we have the total reward $\Phi_N = \sum_{t=0}^{N-1} r'_{\pi}(X_t)$, and the averaged reward $\zeta(r'_{\pi}) = \lim_{N\rightarrow\infty} \frac{1}{N} \mathbb{E}(\Phi_N)$, which can be expressed as $\zeta = \xi r'_{\pi}$. 
Define the limit $\hat{r'_{\pi}} = \lim_{N \rightarrow \infty} \mathbb{E}_{\mu} (\Phi_N - N\zeta)$, which solves the Poisson equation 
$$P\hat{r'_{\pi}} = \hat{r'_{\pi}} - r'_{\pi} + \zeta,$$
where $P$ is the transition matrix and $P(x,y) = p_{\pi}(x,y)$. Two assumptions (\cite{kontoyiannis2003}, Section 4) are needed for the total reward CDF estimation.
\newtheorem{assumption}{Assumption}
\begin{assumption}
	The Markov reward process $X$ is geometrically ergodic with a Lyapunov function $V:X \rightarrow [1,\infty)$ such that $\zeta(V^2)<\infty$.	
\end{assumption}

\begin{assumption} 
	The (measurable) function $r'_{\pi}:S \rightarrow [-1,1]$ has zero mean and nontrivial asymptotic variance $\sigma^2 = \lim_{N \rightarrow \infty} var_x[(\Phi_N)/\sqrt{N}]$.	
\end{assumption}

Under the two assumptions, we show the Edgeworth expansion theorem for nonlattice functionals (Theorem 5.1 in \cite{kontoyiannis2003}) as follows.
\newtheorem{theorem}{Theorem}
\begin{theorem}
	Suppose that $X$ and the strongly nonlattice functional $r'_{\pi}$ satisfy Assumptions 1 and 2, and let $G_N(y)$ denote the distribution function of the normalized partial sums $(\Phi_N - N\zeta(r'_{\pi}))/\sigma\sqrt{N}$:
	$$G_N(y) = \mathbb{P}\{(\Phi_N - N\zeta(r'_{\pi}))/\sigma\sqrt{N} \le y\}, \text{ for all } y \in \mathbb{R}.$$	
	Then, for all $x_0 \in S$ and as $N \rightarrow \infty$,
	$$ G_N(y) = g(y) + \frac{\gamma(y)}{\sigma\sqrt{N}}[\frac{\kappa}{6\sigma^2}(1-y^2)-\hat{r'_{\pi}}(x_0)]+o(N^{-0.5}) \eqno{(2)},$$
	
	where $\gamma(y)$ denotes the standard normal density, $g(y)$ is the corresponding distribution function, and $\kappa$ is a constant related to the third moment of $\Phi_N/\sqrt{N}$. The formulae for $\kappa$, $\hat{r'_{\pi}}$ and $\sigma^2$ can be found in \cite{Xu2011}.
\end{theorem}


\subsection{State-Transition Transformation}

For a Markov reward process derived from an MDP with a transition-based reward function and a stationary policy $\pi$, we cannot apply the method directly since the reward function is $r_{\pi}:S \times S \rightarrow \mathbb{R}$. If the reward function is simplified with (1), the VaR will be affected as illustrated in Section 3. In order to implement the estimation and achieve the real distribution, we propose a method to transform the Markov reward process with a transition-based reward function to one with a state-based reward function. 

\begin{figure}[b]
\vskip 0.2in
\begin{center}
{\includegraphics[width=\columnwidth]{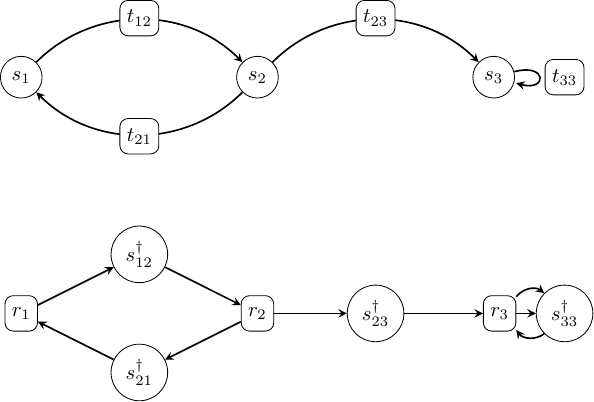}}
\caption{Roles of states and transitions in the transformed Markov reward process.}
\label{MC_Trans}
\end{center}
\vskip -0.2in
\end{figure}

Fig.~\ref{MC_Trans} illustrates what roles the states and transitions play in the original Markov reward process (upper) and its transformed counterpart (lower). In the original Markov reward process, $s_i$ denotes a state and $t_{ij}$ denotes a transition from $s_i$ to $s_j$. In a transformed Markov reward process, the original state $s_i$ becomes a ``router'' $r_i$, which connects input nodes (transformed states) $s^{\dagger}_{ji}$ to output nodes $s^{\dagger}_{ik}$. 

\begin{theorem}
	For a Markov reward process $\langle N, S, r_{\pi}, p_{\pi}, \mu, v\rangle$ with $r_{\pi}$ transition-based, there exists a Markov reward process $\langle N^{\dagger}, S^{\dagger}, r^{\dagger}_{\pi}, p^{\dagger}_{\pi}, \mu^{\dagger}_{\pi}, v^{\dagger}_{\pi}\rangle$ with $r^{\dagger}_{\pi}$ state-based, such that both Markov reward processes have the same total reward distribution.
\end{theorem}

\begin{proof}
	Define a time horizon $N^{\dagger} = N-1$, a state space $S^{\dagger} = S \times S$. For all $x^{\dagger} = (x,y) \in S^{\dagger}$, define a state-based reward function $r^{\dagger}_{\pi}(x^{\dagger}) = r_{\pi}(x,y)$, a salvage function $v_{\pi}^{\dagger}(x^{\dagger}) = v(y)+r_{\pi}(x,y)$, an initial state distribution $\mu_{\pi}^{\dagger}(x^{\dagger})=\mu(x)p_{\pi}(y \mid x)$, a transition kernel $p^{\dagger}_{\pi}(x^{\dagger} \mid y^{\dagger}) = p_{\pi}(y \mid x)$, where $y^{\dagger}=(\cdot,x) \in S^{\dagger}$, otherwise the probability is 0. 
	
	In order to prove that the two Markov reward processes 
	shares the same total reward distribution, consider an arbitrary sample path $(x_0, x_1, \cdots, x_N) \in S^{N+1}$, which corresponds to the sample path $((x_0, x_1), (x_1, x_2), \cdots, (x_{N-1}, x_N)) \in S^{\dagger N}$ in the Markov reward process with a state-based reward function. The probability of the sample path is $\mu(x_0) \times p(x_1 \mid x_0) \times p(x_2 \mid x_1) \times \cdots \times p(x_N \mid x_{N-1}) = \mu_{\pi}^{\dagger}((x_0, x_1)) \times p^{\dagger}_{\pi}((x_1, x_2) \mid (x_0, x_1)) \times \cdots \times p^{\dagger}_{\pi}((x_{N-1}, x_N) \mid (x_{N-2}, x_{N-1}))$, and the corresponding total reward is $r_{\pi}(x_0, x_1) + r_{\pi}(x_1, x_2) + \cdots + r_{\pi}(x_{N-1}, x_N) + v(x_N) = r^{\dagger}_{\pi}((x_0, x_1)) + \cdots + r^{\dagger}_{\pi}((x_{N-2}, x_{N-1})) + v_{\pi}^{\dagger}((x_{N-1}, x_N))$, i.e., the two Markov reward processes have the same total reward distribution.
\end{proof}
\color{black}

Since we take state transitions as states, we name this algorithm \textit{state-transition transformation}, which is described in Algorithm~\ref{alg1}. 
In short, since the distribution estimation only work for the Markov reward processes with state-based reward functions, the original Markov reward process with a transition-based reward function need to be transformed. Comparing with simplifying the reward function with (1), the proposed transformation can keep the distribution intact.
Instead of applying the transformation to a Markov reward process, which derived from an MDP with a policy, we can also apply the transformation directly to the MDP, with a policy-depend initial distribution, as well as a policy-depend salvage reward, if necessary.
\color{black}

\begin{algorithm}[b]
   \caption{State-Transition Transformation}
   \label{alg1}
\begin{algorithmic}
   \STATE {\bfseries Input:} the original Markov reward process $\langle N, S, r_{\pi}, p_{\pi}, \mu, v\rangle$, which is derived from an MDP with a transition-based reward function and a stationary policy $\pi$.
   \STATE {\bfseries Output:} a Markov reward process $\langle N^{\dagger}, S^{\dagger}, r^{\dagger}_{\pi}, p_{\pi}^{\dagger}, \mu_{\pi}^{\dagger}, v^{\dagger}_{\pi}\rangle$.
   \STATE Set the horizon $N^{\dagger} = N-1$;
   \STATE Generate the state space $S^{\dagger} = S \times S$;
   \FOR{{\bfseries all} $x^{\dagger} = (x,y)$ {\bfseries where} $x,y \in S$}
   \STATE Construct the reward function $r^{\dagger}_{\pi}(x^{\dagger}) = r_{\pi}(x,y)$;
   \STATE Construct the transition kernel \\
      		$p^{\dagger}_{\pi}(x^{\dagger} \mid y^{\dagger}) = p_{\pi}(y \mid x)$ for all $y^{\dagger}=(\cdot,x) \in S^{\dagger}$;
   \STATE Set the initial state distribution \\ $\mu_{\pi}^{\dagger}(x^{\dagger}) = \mu(x)p(y \mid x)$;
   \STATE Set the salvage function $v_{\pi}^{\dagger}(x^{\dagger}) = v(y)+r_{\pi}(x,y)$;
   \ENDFOR
\end{algorithmic}
\end{algorithm}


\begin{figure}[b]
\vskip 0.2in
\begin{center}
\centerline{\includegraphics[width=\columnwidth]{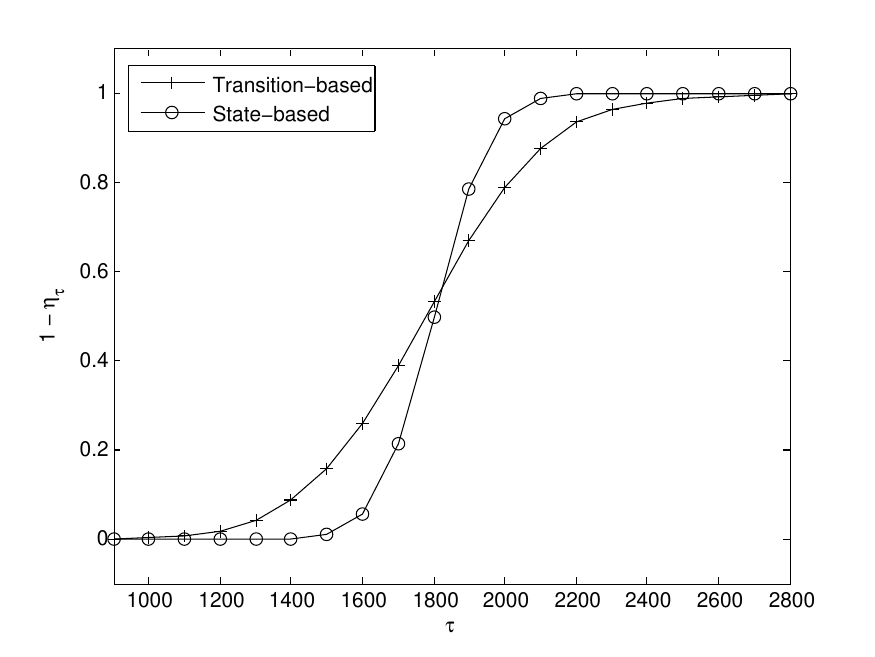}}
\caption{Estimated VaR functions for the long-horizon inventory problems with the transformed reward function and the reward function simplified by (1), respectively.}
\label{figure4}
\end{center}
\vskip -0.2in
\end{figure}

\begin{corollary}
Regardless of the salvage reward, for an MDP $\langle N, S, A, r, p, \mu\rangle$ with a transition-based reward function $r$, denoting any policy as $(\pi_t), t \in {0,\cdots,N-1}$, there exists an MDP $\langle N^{\dagger}, S^{\dagger}, A, r^{\dagger}, p^{\dagger}, \mu^{\dagger}\rangle$ with the state-based reward function $r^{\dagger}$ and the initial state distribution $\mu^{\dagger}((x,y)) = \mu(x)p(y \mid x, \pi_0(x)), x,y \in S$, such that both MDPs have the same total reward distribution when the policy for the second MDP can be denoted as $\pi^{\dagger}((\cdot, x)) = \pi(x)$.
\end{corollary}

In a similar MDP setup outlined in Section III.A, we estimate the VaR function in a long horizon. We set $N=500$ and implement the state-transition transformation to the MDP with the real reward function (transition-based) under a stationary policy. In order to reduce the computational complexity, we use $ \hat{G}_N(y) = g(y)$ instead of (2). In Fig.~\ref{figure4}, we can see that the simplification of the transition-based reward function with (1) results in a distinct VaR function with different values at most nontrivial thresholds ($\tau \in [1300, 1700] \cup [1900,2300]$ in general). For example, when $\tau$ is around 1600, the difference is around 0.2, which may have a serious impact on a risk-sensitive decision making. 


\begin{remark}[VaR in a long\label{short} horizon] 
In a long-horizon MDP, given an estimated VaR function which is strictly increasing, 
for any quantile $\alpha \in [0,1]$, there exists a unique $\tau \in \mathbb{R}$, such that $\alpha = 1-\eta_{\tau}$. In this case, every point along the estimated VaR function can be regarded as a ``jump point'' described in Remark 1. In other words, when the estimated VaR function is strictly increasing, the inversed VaR function solves the VaR Problem 1 with $\alpha \in [0,1]$.
\color{black}
\label{remark2}
\end{remark}

\color{black}

\section{RELATED WORK}
This paper studies two VaR problems defined in \cite{Filar1995b}, which studied the VaR problems on average reward by separating state space into communicating and transient classes. Bouakiz and Kebir \cite{Bouakiz1995} pointed out that the cumulative reward is needed for VaR objectives, and various properties of the optimality equations were studied in both finite and infinite-horizon MDPs. In a finite-horizon MDP, Wu and Lin \cite{Wu1999a} showed that the VaR optimal value functions are target distribution functions, and there exists a deterministic optimal policy. The structure property of optimal policy for an infinite-horizon MDP was also studied. Ohtsubo and Toyonaga \cite{Ohtsubo2002} gave two sufficient conditions for the existence of an optimal policy in infinite-horizon discounted MDPs, and another condition for the unique solution on a transient state set. For the VaR problem with a quantile $\alpha \ge 0.5$, Delage and Mannor \cite{delage2010percentile} solved it as a convex ``second order cone'' programming with reward or transition uncertainty. Different from most studies, Boda and Filar \cite{Boda2006a} and Kira et al. \cite{Kira2012a} considered the VaR objective in a multi-epoch setting, in which a risk measure is required to reach an appropriate level at not only the final epoch but also all intermediate epochs.

The VaR problem with a fixed threshold has been extensively studied. An augmented-state 0-1 MDP was proposed for finite-horizon MDPs with either integer or real-valued reward functions \cite{Xu2011}. By including the cumulative reward space into the state space, the states which satisfied the threshold can be ``tagged'' by a Boolean salvage function. The general reward discretizing error was also bounded \cite{Xu2011}. In a similar problem named MaxProb MDP, the goal states (in which the threshold is satisfied) were defined as absorbing states, and the problem was solved in a similar way \cite{weld2011heuristic}. Value iteration (VI) was proposed to solve the MaxProb MDP in \cite{stella1998optimization}, and followed by some VI variants. In the topological value iteration (TVI) algorithm, states were separated into strongly-connected groups, and efficiency was improved by solving the state groups sequentially \cite{dai2011topological}. Two methods were presented to separate the states efficiently by integrating depth-first search (TVI-DFS) and dynamic programming (TVI-DP) \cite{hou2014revisiting}. For both exact and approximated algorithms for VaR with a threshold, the state of the art can be found in \cite{steinmetz2016goal}.

Constrained probabilistic MDPs take VaR as a constraint. The mean-VaR portfolio optimization problem was solved with the Lagrange multiplier for the VaR constraint over a continuous time span \cite{Yiu2004a}. Bonami and Lejeune \cite{Bonami2009a} solved the mean-variance portfolio optimization problem, and used variants of Chebychev's inequality to derive convex approximations of the VaR function. Randour et al. \cite{Randour2015a} converted the total discounted reward to an almost-sure quantile, and proposed an algorithm based on linear programming to solve the weighted multi-constraint quantile problem. It is also pointed out that randomized policy is necessary when VaR is considered in the constraint, and an example can be found in \cite{Defourny2008a}.

\section{CONCLUSIONS AND DISCUSSIONS}
We studied short- and long-horizon MDPs with finite state spaces considering VaR objectives, and the effect of the simplification of reward function. In short-horizon MDPs, firstly we illustrated that when the real reward function is a transition-based reward function, the simplification with (1) does not affect the optimal value/policy with the expected total reward objective, but changes total reward distribution. 
Secondly we considered two VaR problems, the VaR Problem 2 can be solved with the augmented-state 0-1 MDP method in an expectation way, and we enumerated all policies to obtain the VaR function, which is for both VaR problems. When the horizon is long, we estimate $F^{\pi}_{\Phi}$ for every deterministic policy in order to obtain the VaR function. Since the estimation method is only for Markov reward processes derived from MDPs with state-based reward functions, we propose a transformation algorithm to make it feasible for the MDPs with transition-based reward functions.

Quite a few techniques operate on MDPs with transition-based reward functions, such as Q-learning and the estimation of total reward distribution. However, in many practical scenarios, the reward function is transition-based, and simplified directly with (1) will change the total reward distribution. The state-transition transformation enables the transitions to have properties of states, in order to implement the techniques and keep the distribution intact. We believe that some practical problems with respect to VaR should be revisited using our proposed transformation approach when the reward function is transition-based.

There are several differences between the augmented-state 0-1 MDP and the policy enumeration method. Firstly, since the augmented-state 0-1 MDP is based on the augmented state space, it can only deal with short-horizon MDPs. By contrast, the policy enumeration can work for long-horizon MDPs with the aid of CDF estimation. Secondly, the augmented-state 0-1 MDP can solve the VaR Problem 2 only, but the policy enumeration offers the VaR function, which solves both VaR problems. It is also worth noting that, since the augmented-state 0-1 MDP is based on backward induction (also known as dynamic programming), when the reward function is transition-based, it should not be simplified if the objective is risk-sensitive.

VaR concerns the threshold-quantile pair, and the optimality of one comes into conflict with the other as they are virtually non-increasing functions of each other \cite{Filar1995b}. One future study is to estimate the VaR function without enumerating all the policies. A special case is that there exists an optimal policy $\pi^*$, i.e.,  $F^{\pi^*}_{\Phi}(\tau) = \inf\{F^{\pi}_{\Phi}(\tau)\mid  \pi \in \Pi^N\}\text{ for all } \tau \in \mathbb{R}$. Ohtsubo and Toyonaga \cite{Ohtsubo2002} gave two sufficient conditions for the existence of this optimal policy in infinite-horizon discounted MDPs.
Another idea is to consider it as a dual-objective optimization. Zheng \cite{Zheng2009a} studied the dual-objective MDP concerning variance and CVaR, which might provide some insight. A survey of multi-objective MDP, which concerns more than risk measures, can be found in \cite{Roijers2013}. Furthermore, multiple quantile objectives can also be considered as constraints \cite{Randour2015a}. A study on multiple long-run average objectives can be found in \cite{kuvcera2014markov}.

Considering that many problems have been solved as MDPs with state-based reward functions with risk-sensitive objectives, the error bound of the reward simplification can be another necessary future study, i.e., given an MDP with a state-based reward function simplified from a transition-based one, how far the optimal value/policy can be from the real ones.

\addtolength{\textheight}{-9cm}   

\color{black}

\bibliography{NIPS2016Ref}

\begin{thebibliography}{10}
\providecommand{\url}[1]{#1}
\csname url@rmstyle\endcsname
\providecommand{\newblock}{\relax}
\providecommand{\bibinfo}[2]{#2}
\providecommand\BIBentrySTDinterwordspacing{\spaceskip=0pt\relax}
\providecommand\BIBentryALTinterwordstretchfactor{4}
\providecommand\BIBentryALTinterwordspacing{\spaceskip=\fontdimen2\font plus
\BIBentryALTinterwordstretchfactor\fontdimen3\font minus
  \fontdimen4\font\relax}
\providecommand\BIBforeignlanguage[2]{{%
\expandafter\ifx\csname l@#1\endcsname\relax
\typeout{** WARNING: IEEEtran.bst: No hyphenation pattern has been}%
\typeout{** loaded for the language `#1'. Using the pattern for}%
\typeout{** the default language instead.}%
\else
\language=\csname l@#1\endcsname
\fi
#2}}

\bibitem{altman1999constrained}
E.~Altman, \emph{Constrained Markov decision processes}.\hskip 1em plus 0.5em
  minus 0.4em\relax CRC Press, 1999.

\bibitem{Artzner1998a}
P.~Artzner, F.~Delbaen, J.~Eber, and D.~Heath, ``{Coherent measures of risk},''
  \emph{Mathematical Finance}, vol.~9, no.~3, pp. 1--24, 1998.

\bibitem{Boda2006a}
K.~Boda and J.~A. Filar, ``{Time Consistent Dynamic Risk Measures},''
  \emph{Mathematical Methods of Operations Research}, vol.~63, no.~1, pp.
  169--186, 2006.

\bibitem{Bonami2009a}
P.~Bonami and M.~A. Lejeune, ``{An Exact Solution Approach for Portfolio
  Optimization Problems Under Stochastic and Integer Constraints},''
  \emph{Operations Research}, vol.~57, no.~3, pp. 650--670, 2009.

\bibitem{Bouakiz1995}
M.~Bouakiz and Y.~Kebir, ``{Target-level criterion in Markov decision
  processes},'' \emph{Journal of Optimization Theory and Applications},
  vol.~86, no.~1, pp. 1--15, 1995.

\bibitem{dai2011topological}
P.~Dai, D.~S. Weld, and J.~Goldsmith, ``Topological value iteration
  algorithms,'' \emph{Journal of Artificial Intelligence Research}, vol.~42,
  pp. 181--209, 2011.

\bibitem{Defourny2008a}
B.~Defourny, D.~Ernst, and L.~Wehenkel, ``{Risk-Aware Decision Making and
  Dynamic Programming},'' in \emph{Proceedings of NIPS-08 Workshop on Model
  Uncertainty and Risk in Reinforcement Learning}, 2008, pp. 1--8.

\bibitem{delage2010percentile}
E.~Delage and S.~Mannor, ``Percentile optimization for markov decision
  processes with parameter uncertainty,'' \emph{Operations research}, vol.~58,
  no.~1, pp. 203--213, 2010.

\bibitem{Derman:1970:FSM:578852}
C.~Derman, \emph{{Finite State Markovian Decision Processes}}.\hskip 1em plus
  0.5em minus 0.4em\relax Academic Press, Inc., 1970.

\bibitem{Filar1995b}
J.~A. Filar, D.~Krass, K.~W. Ross, and S.~Member, ``{Percentile Performance
  Criteria For Limiting Average Markov Decision Processes},'' \emph{IEEE
  Transactions on Automatic Control}, vol.~40, no.~I, pp. 2--10, 1995.

\bibitem{hou2014revisiting}
P.~Hou, W.~Yeoh, and P.~Varakantham, ``Revisiting risk-sensitive mdps: New
  algorithms and results,'' in \emph{Proceedings of the International
  Conference on Automated Planning and Scheduling (ICAPS)}, 2014, pp. 136--144.

\bibitem{Kira2012a}
A.~Kira, T.~Ueno, and T.~Fujita, ``{Threshold probability of non-terminal type
  in finite horizon Markov decision processes},'' \emph{Journal of Mathematical
  Analysis and Applications}, vol. 386, no.~1, pp. 461--472, 2012.

\bibitem{weld2011heuristic}
A.~Kolobov, Mausam, D.~S. Weld, and H.~Geffner, ``Heuristic search for
  generalized stochastic shortest path mdps,'' in \emph{Proceedings of the
  International Conference on Automated Planning and Scheduling (ICAPS)}, 2011,
  pp. 130--137.

\bibitem{kontoyiannis2003}
I.~Kontoyiannis and S.~P. Meyn, ``Spectral theory and limit theorems for
  geometrically ergodic markov processes,'' \emph{Annals of Applied
  Probability}, vol.~13, pp. 304--362, 2003.

\bibitem{kuvcera2014markov}
A.~Ku{\v{c}}era, V.~Forejt, K.~Chatterjee, V.~Bro{\v{z}}ek, and T.~Br{\'a}zdil,
  ``Markov decision processes with multiple long-run average objectives,''
  \emph{Logical Methods in Computer Science}, vol.~10, 2014.

\bibitem{Mannor2011a}
S.~Mannor and J.~Tsitsiklis, ``{Mean-Variance Optimization in Markov Decision
  Processes},'' in \emph{Proceedings of the 28th International Conference on
  Machine Learning (ICML)}, 2011, pp. 1--22.

\bibitem{Ohtsubo2002}
Y.~Ohtsubo and K.~Toyonaga, ``{Optimal policy for minimizing risk models in
  Markov decision processes},'' \emph{Journal of mathematical analysis and
  applications}, vol. 271(1), pp. 66--81, 2002.

\bibitem{Prashanth2015}
L.~A. Prashanth, C.~Jie, M.~Fu, S.~Marcus, and C.~Szepesv{\'a}ri, ``Cumulative
  prospect theory meets reinforcement learning: Prediction and control,'' in
  \emph{Proceedings of the 33rd International Conference on Machine Learning,
  New York, NY, USA}, 2016, pp. 1406--1415.

\bibitem{Puterman1994a}
M.~L. Puterman, \emph{{Markov Decision Processes: Discrete Stochastic Dynamic
  Programming}}.\hskip 1em plus 0.5em minus 0.4em\relax Wiley, 1994.

\bibitem{Randour2015a}
M.~Randour, J.~Raskin, and O.~Sankur, ``{Percentile Queries in
  Multi-dimensional Markov Decision Processes},'' \emph{Computer Aided
  Verification}, vol. 9206, pp. 123--139, 2015.

\bibitem{Riedel2004}
F.~Riedel, ``{Dynamic coherent risk measures},'' \emph{Stochastic Processes and
  their Applications}, vol. 112, no.~2, pp. 185--200, 2004.

\bibitem{Roijers2013}
D.~M. Roijers, P.~Vamplew, S.~Whiteson, and R.~Dazeley, ``{A Survey of
  Multi-Objective Sequential Decision-Making},'' \emph{Journal of Artificial
  Intelligence Research}, vol.~48, pp. 67--113, 2013.

\bibitem{Ruszczynski2006a}
A.~Ruszczy{\'{n}}ski and A.~Shapiro, ``{Optimization of Convex Risk
  Functions},'' \emph{Mathematics of Operations Research}, vol.~31, no.~3, pp.
  433--452, 2006.

\bibitem{doi:10.1287/opre.42.1.175}
M.~J. Sobel, ``{Mean-Variance Tradeoffs in an Undiscounted MDP},''
  \emph{Operations Research}, vol.~42, no.~1, pp. 175--183, 1994.

\bibitem{steinmetz2016goal}
M.~Steinmetz, J.~Hoffmann, and O.~Buffet, ``Goal probability analysis in mdp
  probabilistic planning: Exploring and enhancing the state of the art,''
  \emph{Journal of Artificial Intelligence Research}, vol.~57, pp. 229--271,
  2016.

\bibitem{watkins1992q}
C.~Watkins and P.~Dayan, ``Q-learning,'' \emph{Machine learning}, vol.~8, no.
  3-4, pp. 279--292, 1992.

\bibitem{D.J.White1988a}
{White, D. J.}, ``{Mean , Variance , and Probabilistic Criteria in Finite
  Markov Decision Processes : A Review},'' \emph{Journal of Optimization Theory
  and Applications}, vol.~56, no.~1, pp. 1--29, 1988.

\bibitem{Wu1999a}
C.~Wu and Y.~Lin, ``{Minimizing Risk models in Markov decision process with
  policies depending on target values},'' \emph{Journal of Mathematical
  Analysis and Applications}, vol.~23, no.~1, pp. 47--67, 1999.

\bibitem{Xu2011}
H.~Xu and S.~Mannor, ``{Probabilistic goal Markov decision processes},'' in
  \emph{Proceedings of the International Joint Conference on Artificial
  Intelligence (IJCAI)}, 2011, pp. 2046--2052.

\bibitem{Yiu2004a}
K.~F.~C. Yiu, S.~Y. Wang, and K.~L. Mak, ``{Optimal portfolios under a
  value-at-risk constraint},'' \emph{Journal of Economic Dynamics and Control},
  vol.~28, no.~7, pp. 1317--1334, 2004.

\bibitem{stella1998optimization}
S.~X. Yu, Y.~Lin, and P.~Yan, ``Optimization models for the first arrival
  target distribution function in discrete time,'' \emph{Journal of
  mathematical analysis and applications}, vol. 225, no.~1, pp. 193--223, 1998.

\bibitem{Zheng2009a}
H.~Zheng, ``{Efficient frontier of utility and CVaR},'' \emph{Mathematical
  Methods of Operations Research}, vol.~70, no.~1, pp. 129--148, 2009.

\end{thebibliography}
\bibliographystyle{IEEEtranS}

\newpage
\newpage

\end{document}